\documentclass{article}

\PassOptionsToPackage{round}{natbib}

\usepackage[preprint]{neurips_2023}

\usepackage[american]{babel}

\usepackage{microtype}
\usepackage{graphicx}
\usepackage{subfigure}
\usepackage{multirow}
\usepackage{bbm}

\usepackage{hyperref}

\usepackage[dvipsnames]{xcolor}

\usepackage{mathtools} 
\usepackage{booktabs} 
\usepackage{tikz} 

\usepackage{amsmath,amsfonts,bm}

\def\eqref#1{equation~\ref{#1}}

\def\1{\bm{1}}

\DeclareMathAlphabet{\mathsfit}{\encodingdefault}{\sfdefault}{m}{sl}
\SetMathAlphabet{\mathsfit}{bold}{\encodingdefault}{\sfdefault}{bx}{n}

\def\gA{{\mathcal{A}}}
\def\gB{{\mathcal{B}}}

\def\gD{{\mathcal{D}}}

\def\gG{{\mathcal{G}}}

\def\gS{{\mathcal{S}}}
\def\gT{{\mathcal{T}}}

\def\gX{{\mathcal{X}}}

\def\sN{{\mathbb{N}}}

\def\sP{{\mathbb{P}}}

\def\sR{{\mathbb{R}}}

\newcommand{\E}{\mathbb{E}}

\makeatletter
\DeclareRobustCommand{\cev}[1]{%
  \mathpalette\do@cev{#1}%
}
\newcommand{\do@cev}[2]{%
  \fix@cev{#1}{+}%
  \reflectbox{$\m@th#1\vec{\reflectbox{$\fix@cev{#1}{-}\m@th#1#2\fix@cev{#1}{+}$}}$}%
  \fix@cev{#1}{-}%
}
\newcommand{\fix@cev}[2]{%
  \ifx#1\displaystyle
    \mkern#23mu
  \else
    \ifx#1\textstyle
      \mkern#23mu
    \else
      \ifx#1\scriptstyle
        \mkern#22mu
      \else
        \mkern#22mu
      \fi
    \fi
  \fi
}

\makeatother

\usepackage{hyperref}
\usepackage{url}
\usepackage{amsmath}
\usepackage{amssymb}
\usepackage{mathtools}
\usepackage{amsthm}
\usepackage{lipsum}
\usepackage{thmtools}
\usepackage{thm-restate}
\usepackage{adjustbox}

\usepackage[capitalize,noabbrev]{cleveref}

\definecolor{mydarkblue}{rgb}{0,0.08,0.45}
\hypersetup{ %
    pdftitle={},
    pdfauthor={},
    pdfsubject={},
    pdfkeywords={},
    pdfborder=0 0 0,
    pdfpagemode=UseNone,
    colorlinks=true,
    linkcolor=mydarkblue,
    citecolor=mydarkblue,
    filecolor=mydarkblue,
    urlcolor=mydarkblue,
}

\usepackage[framemethod=TikZ]{mdframed}
\theoremstyle{plain}

\declaretheoremstyle[
    headfont=\bfseries, 
    bodyfont=\normalfont\itshape,
    headpunct={.},
    postheadspace=0.5em,
    spacebelow=\parsep,
    spaceabove=\parsep,
    mdframed={
        roundcorner = 5pt,%
        backgroundcolor = gray!10,%
        innertopmargin = \topskip,%
        splittopskip = \topskip,%
        hidealllines = true,%
    } 
]{framedstyle}
\declaretheorem[parent=section, style=framedstyle, name=Theorem]{ftheorem}

\newtheorem{theorem}[ftheorem]{Theorem}
\newtheorem{proposition}[theorem]{Proposition}
\newtheorem{lemma}[theorem]{Lemma}

\theoremstyle{definition}
\newtheorem{definition}[theorem]{Definition}

\theoremstyle{remark}

\crefformat{section}{#2Section~#1#3}
\crefformat{appendix}{#2Appendix~#1#3}

\crefformat{equation}{(#2#1#3)}
\crefmultiformat{equation}{#2Equations~#1#3}%
{ \&~#2#1#3}{, #2#1#3}{, \&~#2#1#3}

\crefformat{table}{#2Table~#1#3}
\crefformat{figure}{#2Figure~#1#3}

\crefformat{ftheorem}{#2Theorem~#1#3}
\crefmultiformat{ftheorem}{#2Theorems~#1#3}%
{ \&~#2#1#3}{, #2#1#3}{, and~(#2#1#3)}

\crefformat{theorem}{#2Theorem~#1#3}
\crefmultiformat{theorem}{#2Theorems~#1#3}%
{ \&~#2#1#3}{, #2#1#3}{, and~(#2#1#3)}

\crefformat{lemma}{#2Lemma~#1#3}
\crefformat{proposition}{#2Proposition~#1#3}
\crefmultiformat{proposition}{#2Propositions~#1#3}%
{ \&~#2#1#3}{, #2#1#3}{, and~(#2#1#3)}

\crefformat{algorithm}{#2Algorithm~#1#3}
\crefmultiformat{algorithm}{#2Algorithms~#1#3}%
{ \&~#2#1#3}{, #2#1#3}{, and~(#2#1#3)}

\crefformat{corollary}{#2Corollary~#1#3}
\crefformat{definition}{#2Definition~#1#3}

\crefformat{assumption}{#2Assumption~#1#3}
\crefmultiformat{assumption}{#2Assumptions~#1#3}%
{ \&~#2#1#3}{, #2#1#3}{, and~(#2#1#3)}

\newcommand{\children}{\mathrm{Ch}}
\newcommand{\Pa}{\mathrm{Pa}}

\usepackage{pifont}

\title{Generative Flow Networks:\\a Markov Chain Perspective}

\makeatletter
\def\thanks#1{\protected@xdef\@thanks{\@thanks
        \protect\footnotetext{#1}}}
\makeatother

\author{
    Tristan~Deleu{\normalfont\textsuperscript{1}}\thanks{\textsuperscript{1}Universit\'{e} de Montr\'{e}al, \textsuperscript{2}CIFAR Senior Fellow.}\thanks{\textsuperscript{\phantom{1}}Correspondence: Tristan Deleu (\href{mailto:deleutri@mila.quebec}{deleutri@mila.quebec})}\qquad Yoshua~Bengio{\normalfont\textsuperscript{1,2}}\\[1ex]
    Mila -- Quebec AI Institute
}

\begin{document}

\maketitle

\begin{abstract}
    While Markov chain Monte Carlo methods (MCMC) provide a general framework to sample from a probability distribution defined up to normalization, they often suffer from slow convergence to the target distribution when the latter is highly multi-modal. Recently, Generative Flow Networks (GFlowNets) have been proposed as an alternative framework to mitigate this issue when samples have a clear compositional structure, by treating sampling as a sequential decision making problem. Although they were initially introduced from the perspective of flow networks, the recent advances of GFlowNets draw more and more inspiration from the Markov chain literature, bypassing completely the need for flows. In this paper, we formalize this connection and offer a new perspective for GFlowNets using Markov chains, showing a unifying view for GFlowNets regardless of the nature of the state space as recurrent Markov chains. Positioning GFlowNets under the same theoretical framework as MCMC methods also allows us to identify the similarities between both frameworks, and most importantly to highlight their differences.
\end{abstract}

\section{Introduction}
\label{sec:introduction}
Sampling from a probability distribution defined up to a normalization constant $p(x) = R(x) / Z$ can be a challenging problem when the normalization constant $Z$ is intractable (e.g., the partition function of an energy-based model over a large, or even continuous, sample space), and only the unnormalized probability $R(x)$ can be computed easily. Besides energy-based models, this also includes posterior distributions of the form $p(x\mid \gD)$ in Bayesian inference, that is the joint distribution $p(\gD\mid x)p(x)$ (which can often be evaluated analytically), normalized by the evidence $p(\gD)$, which is typically intractable as well. In those cases, Markov chain Monte Carlo methods (MCMC; \citealp{neal1993probabilistic}) have proven to be a versatile tool to sample from such distributions defined up to normalization. Despite their generality though, MCMC methods may suffer from slow mixing when the target distribution $p(x)$ is highly multi-modal, affecting the convergence of these methods and therefore yielding poor samples which are not representative of $p(x)$.

Recently, \emph{Generative Flow Networks} (GFlowNets; \citealp{bengio2021gflownet,bengio2021gflownetfoundations}) have been proposed as an alternative framework to sample from distributions defined up to normalization, when the samples in question have a natural compositional structure. Applications of GFlowNets include generating small molecules \citep{bengio2021gflownet}, Bayesian structure learning of Bayesian Networks \citep{deleu2022daggflownet,deleu2023jspgfn}, modeling Bayesian posteriors over structured latent variable models \citep{hu2023gfnem}, generating biological sequences \citep{jain2022gfnbiological}, as well as scientific discovery at large \citep{jain2023gfnscientific}. Unlike MCMC, GFlowNets treat the generation of a sample not as a Markov chain over the sample space, but as a sequential decision making problem where each new sample is constructed piece by piece from scratch, mitigating the problem of sampling from diverse modes of a multi-modal target distribution. \citet{lahlou2023continuousgfn} recently extended this framework to more general state spaces, including continuous spaces.

Even though GFlowNets were originally introduced from the perspective of \emph{flow networks} \citep{bengio2021gflownet}, most of the recent advances completely ignore the need to work with flows, and work directly as a Markovian process, with conditions heavily inspired by the literature on Markov chains (e.g., the detailed balance conditions; \citealp{bengio2021gflownetfoundations}). In this work, we formalize this connection with Markov chains by treating a GFlowNet as a recurrent Markov chain, whose certain marginal distribution matches the target distribution under some boundary conditions. This new perspective under the same theoretical framework as MCMC methods allows us to better understand the similarities between both methods, but also to identify clearly their differences. The objective of this paper is also to provide an introduction to GFlowNets for researchers with expertise in MCMC methods, using a familiar approach and a similar vocabulary.

\section{Discrete Generative Flow Networks}
\label{sec:discrete-gflownet}
In this section, we recall the formalism of GFlowNets as a pointed Directed Acyclic Graph (DAG), as described in \citet{bengio2021gflownetfoundations}, and introduce a new perspective using recurrent Markov chains. For some DAG $\gG$, we will use the notations $\Pa_{\gG}(s)$ and $\children_{\gG}(s)$ to denote the set of parents and children of $s$ respectively in $\gG$.

\subsection{Flow networks over pointed Directed Acyclic Graphs}
\label{sec:flow-networks-pointed-dags}
A \emph{Generative Flow Network} (GFlowNet) is a generative model that treats the generation of an object as a sequential decision making problem \citep{bengio2021gflownet,bengio2021gflownetfoundations}. We assume in this section that these objects $x \in \gX$ are discrete and have some compositional structure, where the sample space is denoted by $\gX$. A GFlowNet constructs a sample $x \in \gX$ by creating a process over a (finite) superset of states $\gS \supseteq \gX$, starting at some fixed \emph{initial state} $s_{0} \in \gS$, leveraging the compositional structure of $x$. For example, \citet{bengio2021gflownet} used a GFlowNet to define a distribution over molecules, where $\gX$ represents the space of all (complete) molecules, which can be constructed piece by piece by attaching a new fragment to a partially constructed molecule, i.e. a state in $\gS \backslash \gX$, starting from the empty state $s_{0}$. Beyond containing the sample space $\gX$, the states of $\gS$ therefore serve as intermediate steps along the generation process.

\citet{bengio2021gflownetfoundations} formalized this process by structuring the state space $\gS$ as a Directed Acyclic Graph (DAG) $\gG = (\bar{\gS}, \gA)$, whose vertices $\bar{\gS} = \gS \cup \{s_{f}\}$ correspond to the states of the GFlowNet $\gS$, with the addition of an abstract state $s_{f} \notin \gS$ with no child, called the \emph{terminal state}. The terminal state is added in such a way that its parents are exactly the elements of the sample space: $\Pa_{\gG}(s_{f}) = \gX$. The states $x\in\gX$ are called \emph{terminating states}. The edges of the DAG follow the compositional structure of the states in $\gS$, while guaranteeing acyclicity; for example, there may be an edge $s \rightarrow s' \in \gA$ if $s'$ is the result of adding a new fragment to a partial molecule $s$. We also assume that all the states $s\in\gS$ are accessible from the initial state $s_{0}$, and that $\gG$ is a \emph{pointed DAG}, meaning that $\gG$ is rooted at $s_{0}$ and any trajectory starting at $s_{0}$ following the edges in $\gA$ eventually terminates at $s_{f}$. See \cref{fig:discrete-gflownet} (left) for an example of a pointed DAG.

In addition to the pointed DAG structure over states, every terminating state $x\in\gX$ is associated with a \emph{reward} $R(x) > 0$, indicating a notion of ``preference'' for certain states. By convention, we set $R(s) = 0$ for any intermediate state $s \in \gS \backslash \gX$. The goal of a GFlowNet is to find a \emph{flow} $F$ along the edges of $\gG$ that satisfy, for all the states $s' \in \gS \backslash \{s_{0}\}$, the following \emph{flow-matching condition}:
\begin{equation}
    \sum_{\mathclap{s \in \Pa_{\gG}(s')}}\;F(s \rightarrow s') -\;\sum_{\mathclap{s'' \in \children_{\gG}(s')}}\;F(s' \rightarrow s'') = R(s').
    \label{eq:flow-matching-condition}
\end{equation}
The condition above has an intuitive interpretation: we want the total amount of flow going into $s'$ to be equal to the total amount of flow going out of $s'$, with some residual $R(s')$. If \cref{eq:flow-matching-condition} is satisfied for all $s' \in \gS \backslash \{s_{0}\}$, then the GFlowNet induces a distribution over the terminating states $x \in \gX$ which is proportional to $R(x)$. More precisely, if we sample a complete trajectory $(s_{0}, s_{1}, \ldots, s_{T}, x, s_{f})$ using a transition probability distribution defined by normalizing the outgoing flows
\begin{equation}
    P_{F}(s_{t+1} \mid s_{t}) \propto F(s_{t} \rightarrow s_{t+1}),
    \label{eq:forward-transition-probability}
\end{equation}
with the conventions $s_{T+1} = x$ and $s_{T+2} = s_{f}$, then $x$ is sampled with probability proportional to $R(x)$. The distribution induced by the GFlowNet $P_{F}^{\top}(x) \propto R(x)$ is called the \emph{terminating state probability distribution} \citep{bengio2021gflownetfoundations}, and corresponds to the marginal distribution of the process described above at $x\in\gX$:
\begin{equation}
    P_{F}^{\top}(x) \triangleq P_{F}(x\rightarrow s_{f})\sum_{\tau: s_{0} \rightsquigarrow x}\prod_{t=0}^{T_{\tau}}P_{F}(s_{t+1}\mid s_{t}),
\end{equation}
where $s_{0} \rightsquigarrow x$ denotes all the possible trajectories $\tau$ from $s_{0}$ to $x$, following the edges in $\gA$. Independent samples of $P_{F}^{\top}$ can be obtained by running the above process multiple times, with completely independent trajectories.

\subsection{The cyclic structure of GFlowNets}
\label{sec:cyclical-discrete-gflownet}

\begin{figure}[t]
    \centering
    \includegraphics[]{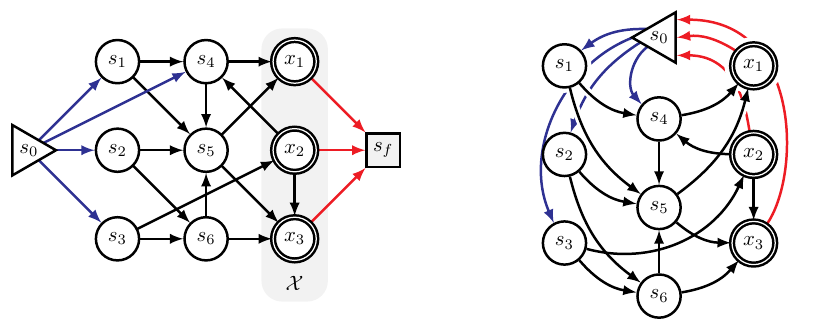}
    \caption{A discrete GFlowNet as a recurrent Markov chain. (Left) The pointed DAG structure of a GFlowNet, with the flow going from the initial state $s_{0}$ to the terminal state $s_{f}$. The terminating states $x\in \gX$ are represented with double circles. (Right) The ``wrapped around'' transformation of the same GFlowNet, where the initial state has been merged with the terminal state $s_{f} \equiv s_{0}$.}
    \label{fig:discrete-gflownet}
\end{figure}

Instead of treating independent samples of the GFlowNet as being completely separate trajectories in a pointed DAG, we can view this process as being a \emph{single} Markov chain that regenerates every time it reaches $s_{0}$. We consider a Markov chain over a discrete (but not necessarily finite) state space $\gS$, with transition probability $P_{F}$. We assume that this Markov chain is \emph{irreducible}, to ensure that the Markov chain reaches any state $s\in\gS$ from the initial state, which is a necessary condition of GFlowNets. This may be enforced by requiring $P_{F}$ to follow an appropriate directed graph structure over the state space. Furthermore, we assume that the Markov chain is \emph{positive recurrent}, to guarantee the existence of an invariant measure $F$:
\begin{equation}
    \forall s' \in \gS, \qquad F(s') = \sum_{s\in \gS}F(s)P_{F}(s, s').
    \label{eq:invariant-measure}
\end{equation}
The invariant measure $F$ plays the role of the \emph{state flow} in GFlowNets \citep{bengio2021gflownetfoundations}, so that the product $F(s)P_{F}(s, s')$ would represent the edge flow in \cref{sec:flow-networks-pointed-dags} and \cref{eq:invariant-measure} can be viewed as an alternative way to write the flow-matching conditions \cref{eq:flow-matching-condition} in a GFlowNet, apart from the residual part in $R(s')$. Recurrence can be achieved by ``wrapping around'' the structure of the GFlowNet, and merging the terminal state $s_{f}$ with the initial state $s_{0}$. An example of this construction from a pointed DAG is shown in \cref{fig:discrete-gflownet}. Unlike recurrence though, the graph structure alone combined with $s_{f}\equiv s_{0}$ is not enough to conclude that the Markov chain is positive. However, positiveness here is only required when $\gS$ is infinite, since any recurrent Markov chain over a finite state space is necessarily positive (and therefore admits an invariant measure); see \cref{app:counter-example-positive-recurrence} for a detailed counter-example. The invariant measure of such a Markov chain is essentially unique, up to a normalization constant (see \cref{thm:unique-invariant-measure}).

\paragraph{Terminating state probability.} The fact that the Markov chain is recurrent guarantees that it will return back to its initial state $s_{0}$ an infinite amount of time. We can define the \emph{return time} $\sigma_{s_{0}}$ as being the (random) time the chain first goes back to the initial state:
\begin{equation}
    \sigma_{s_{0}} = \inf \{k \geq 1\mid X_{k} = s_{0}\},
    \label{eq:return-time}
\end{equation}
where $(X_{k})_{k\geq 0}$ is the (canonical) Markov chain following the transition probability $P_{F}$ and such that $X_{0} = s_{0}$. Excursions of the Markov chain between two consecutive returns to $s_{0}$ are independent by the strong Markov property, since $\sigma_{s_{0}}$ is a stopping time; in other words, the independence of trajectories for two different samples in a GFlowNet is preserved even with a single recurrent Markov chain running indefinitely. We can define the \emph{terminating state probability} from \cref{sec:flow-networks-pointed-dags} in terms of the return time of the Markov chain.

\begin{definition}[Terminating state probability]
    Given an irreducible and positive recurrent Markov chain over $\gS$, with transition probability $P_{F}$, the \emph{terminating state probability distribution} is defined as the marginal distribution of the chain right before returning to its initial state $s_{0}$:
    \begin{equation}
        P_{F}^{\top}(x) \triangleq \sP_{s_{0}}(X_{\sigma_{s_{0}}-1} = x) = \E_{s_{0}}\big[\mathbbm{1}(X_{\sigma_{s_{0}}-1}=x)\big]
        \label{eq:terminating-state-probability}
    \end{equation}
    \label{def:terminating-state-probability}
\end{definition}

We show in \cref{prop:terminating-state-probability-proper-distribution} that $P_{F}^{\top}$ is a properly defined probability distribution (i.e., non-negative and sums to 1) over $\gX\subseteq \gS$, corresponding to the parents of the initial state $s_{0}$ in the directed graph structure described above: $\gX = \{s\in\gS \mid P_{F}(s, s_{0}) > 0\}$.

\paragraph{Boundary conditions.} The goal of a GFlowNet is to find a transition probability $P_{F}$ such that the corresponding terminating state probability $P_{F}^{\top}$ matches the reward function, up to normalization. Similar to \cref{sec:flow-networks-pointed-dags}, this requires satisfying \emph{boundary conditions} of the form
\begin{equation}
    \forall x\in \gX, \qquad F(x)P_{F}(x, s_{0}) = R(x),
    \label{eq:boundary-condition}
\end{equation}
in addition to $F$ being an invariant measure for $P_{F}$ (i.e., the ``flow-matching conditions''). Using the interpretation of $F(x)P_{F}(x, s_{0})$ as being the flow through $x \rightarrow s_{0}$ as mentioned above, these boundary conditions are equivalent to the reward matching conditions enforced in GFlowNets \citep{bengio2021gflownetfoundations}. If $P_{F}$ admits an invariant measure $F$ that satisfies the boundary conditions above, then we get the fundamental theorem of GFlowNets that guarantees that the terminating state probability distribution induced by the GFlowNet is proportional to the reward.
\vspace*{2\parsep}
\begin{restatable}{ftheorem}{invariantterminatingstate}
    Let $P_{F}$ be an irreducible and positive recurrent Markov kernel over $\gS$ that admits an invariant measure $F$ such that $\forall x\in\gX,\ F(x)P_{F}(x, s_{0}) = R(x)$, where $R$ is a finite measure on $\gX \subseteq \gS$. Then the terminating state probability distribution defined in \cref{def:terminating-state-probability} is proportional to the measure $R$:
    \begin{equation*}
        \forall x \in \gX, \qquad P_{F}^{\top}(x) = \frac{R(x)}{\sum_{x'\in\gX}R(x')}
    \end{equation*}
    \label{thm:invariant-terminating-state}
\end{restatable}

The proof of this theorem is based on the unicity of the invariant measure of $P_{F}$ up to a normalization constant, and is available in \cref{app:discrete-spaces}.

In practice, the objective is to find a transition probability $P_{F}$ and a measure $F$, which may be parametrized by neural networks, that satisfy the conditions of \cref{thm:invariant-terminating-state}: $F$ must satisfy the boundary conditions of \cref{eq:boundary-condition}, and must also be invariant for $P_{F}$. The summation in \cref{eq:invariant-measure} is typically inexpensive if the state space is well structured (i.e., if the objects have a clear compositional nature), since we only have to sum over the parents of $s'$---all the states $s\in \gS$ such that $P_{F}(s, s') > 0$. However, checking the invariance of $F$ at the initial state $s_{0}$ proves to be as difficult as computing the partition function itself, since the parent set of $s_{0}$ is the whole sample space $\gX$. Fortunately, the following proposition shows that we only have to check that $F$ satisfies \cref{eq:invariant-measure} for any state $s' \neq s_{0}$.

\begin{restatable}{proposition}{invariantmeasurenoinitialstate}
    Let $P_{F}$ be an irreducible and positive recurrent Markov kernel over $\gS$. A measure $F$ is invariant for $P_{F}$ if and only if for all $s' \neq s_{0}$ we have
    \begin{equation}
        F(s') = \sum_{s\in\gS}F(s)P_{F}(s, s').
        \label{eq:invariant-measure-no-initial-state}
    \end{equation}
    \label{prop:invariant-measure-no-initial-state}
\end{restatable}

The proof is available in \cref{app:discrete-spaces}. This result mirrors the practical implementation of GFlowNets as a pointed DAG, where the flow-matching conditions are never checked at the terminal state $s_{f}$, since it corresponds to an abstract state that is not in $\gS$. The invariance of $F$ over the whole state space is only a convenient tool that allows us to use existing results from the Markov chain literature, which can be extended to general state spaces.

\section{Generative Flow Networks over General State Spaces}
\label{sec:general-gflownet}

\subsection{Existing extensions of GFlowNets beyond discrete state spaces}
\label{sec:existing-continuous-gflownets}
While GFlowNets were initially introduced to construct distributions over discrete objects \citep{bengio2021gflownet}, some existing works have proposed to generalize this framework beyond discrete state spaces. For example in CFlowNets \citep{li2023cflownets}, the authors considered the case where $\gS$ is a continuous space, and introduced a flow-matching condition where the summations in \ref{eq:flow-matching-condition} were simply replaced by integrals. As highlighted by \citet{lahlou2023continuousgfn} though, implicit assumptions made on the transition function and the omission of critical aspects of GFlowNets, such as the accessibility of any state $s\in\gS$ from the initial state $s_{0}$, severely limit the scope of applications of CFlowNets.

Closely related to our work, \citet{lahlou2023continuousgfn} introduced a theoretical framework for studying GFlowNets in general state spaces, thus including continuous state spaces, and even hybrid spaces with both discrete and continuous components \citep{deleu2023jspgfn}. For a measurable space $(\gS, \Sigma)$, where $\Sigma$ is a $\sigma$-algebra on $\gS$, their approach relies on the notion of \emph{Markov kernels}, which generalizes transition probabilities in discrete spaces.

\begin{definition}[Markov kernel]
    Let $(\gS, \Sigma)$ be a measurable state space. A function $\kappa: \gS \times \Sigma \rightarrow [0, +\infty)$ is called a positive $\sigma$-finite \emph{transition kernel} if
    \begin{enumerate}
        \item For any $B \in \Sigma$, the mapping $s\mapsto \kappa(s, B)$ is measurable, where the space $[0, +\infty)$ is associated with the Borel $\sigma$-algebra $\gB([0, +\infty))$;
        \item For any $s\in\gS$, the mapping $B \mapsto \kappa(s, B)$ is a positive $\sigma$-finite measure on $(\gS, \Sigma)$.
    \end{enumerate}
    Furthermore, if the mappings $\kappa(s, \cdot)$ are probability distributions (i.e., $\kappa(s, \gS) = 1$), the transition kernel is called a \emph{Markov kernel}.
\end{definition}

Taking inspiration from the pointed DAG formulation \citep{bengio2021gflownetfoundations}, \citet{lahlou2023continuousgfn} augmented the state space $\bar{\gS} = \gS \cup \{\bot\}$ with a distinguished element $\bot \notin \gS$, and proposed a generalization of the pointed DAG structure to measurable spaces, called a ``measurable pointed graph'', which is also ``finitely absorbing'' in the sense that any Markov chain starting at $s_{0}$ eventually reaches $\bot$ in bounded time. Unlike CFlowNets \citep{li2023cflownets}, which were still operating on edge flows directly as in \cref{sec:flow-networks-pointed-dags}, they defined a flow on this measurable pointed graph as a tuple $F = (\mu, \bar{P}_{F})$ satisfying the following flow-matching conditions
\begin{equation}
    \int_{\bar{\gS}}f(s')\mu(ds') = \iint_{\gS\times \bar{\gS}}f(s')\mu(ds)\bar{P}_{F}(s, ds'),
    \label{eq:lahlou-flow-matching-conditions}
\end{equation}
for any measurable function $f: \bar{S}\rightarrow \sR$ such that $f(s_{0}) = 0$, where $\bar{P}_{F}$ is a Markov kernel\footnote{We use the notation $\bar{P}_{F}$ for the Markov kernel in \citet{lahlou2023continuousgfn}, to avoid confusion with the rest of the paper where we use another Markov kernel $P_{F}$, which in particular will be defined over $\gS$ and not $\bar{\gS}$.} on $(\bar{\gS}, \bar{\Sigma})$ and $\mu$ is a measure over $\gS$ ($\bar{\Sigma}$ being the augmented $\sigma$-algebra associated to $\bar{\gS}$, built from $\Sigma$).

\subsection{Harris recurrence and invariant measures}
\label{sec:harris-recurrence-invariant-measures}
We consider a Markov kernel $P_{F}$ on a measurable space $(\gS, \Sigma)$. For a fixed measure $\phi$ on $\gS$, we will assume that $P_{F}$ is \emph{$\phi$-irreducible}, meaning that any set $B \in \Sigma$ such that $\phi(B) > 0$ is accessible from any state in $\gS$ (see \cref{def:phi-irreducibility} for details). Similar to the discrete case, we will also assume some form of recurrence for the Markov kernel in addition to irreducibility, in order to guarantee the existence of an invariant measure for $P_{F}$. For general state spaces, we will use a stronger notion of recurrence called \emph{Harris recurrence}.

\begin{definition}[Harris recurrence]
    A $\phi$-irreducible Markov kernel $P_{F}$ is said to be \emph{Harris recurrent} if for all set $B \in \Sigma$ such that $\phi(B) > 0$, any Markov chain starting in $B$ eventually returns back to $B$ in finite time with probability 1:
    \begin{equation}
        \forall s \in B,\qquad \sP_{s}(\sigma_{B} < \infty) = 1,
    \end{equation}
    where $\sigma_{B} = \inf \{k \geq 1\mid X_{k}\in B\}$ is the return time of the Markov chain to $B$ (similar to \cref{eq:return-time}).
    \label{def:harris-recurrence}
\end{definition}

The condition that the Markov chain returns to any accessible set with probability 1 is reminiscent of the ``finitely absorbing'' condition of \citet{lahlou2023continuousgfn} that requires all trajectories to be of bounded length (see also \cref{prop:finitely-absorbing-harris}). The fact that $P_{F}$ is Harris recurrent ensures that there exists an invariant measure $F$ over $\gS$ \citep{douc2018markovchains}, i.e., for any bounded measurable function $f: \gS \rightarrow \sR$, we have
\begin{equation}
    \int_{\gS}f(s')F(ds') = \iint_{\gS\times\gS} f(s')F(ds)P_{F}(s, ds').
    \label{eq:invariant-measure-general}
\end{equation}
The equation above is similar to the flow-matching condition \cref{eq:lahlou-flow-matching-conditions} in \citet{lahlou2023continuousgfn}, with the exception that there is no restriction of the form $f(s_{0}) = 0$, and integration is carried out on the same (non-augmented) space $\gS$. In fact, the flow-matching condition \cref{eq:lahlou-flow-matching-conditions} encodes a form of invariance of $\mu$ for $P_{F}$ everywhere \emph{except at $s_{0}$}, thus following closely the conditions in \cref{eq:flow-matching-condition} for discrete state spaces.

\subsection{Creation of an atom via the splitting technique}
The key property preserved by wrapping around the state space at $s_{0}\equiv s_{f}$ in \cref{sec:cyclical-discrete-gflownet} was that the Markov chain was effectively ``regenerating'' every time it was returning to $s_{0}$, thanks to the (strong) Markov property. In this context, $\{s_{0}\}$ is called an \emph{atom} of the Markov chain (\citealp{meyn1993markovchainsstability}), which informally corresponds to the chain ``forgetting'' about the past every time it goes through any state in the atom.

\begin{definition}[Atom]
    Let $P_{F}$ be a Markov kernel over a measurable space $(\gS, \Sigma)$. A set $A \in \Sigma$ is called an \emph{atom} if there exists a probability measure $\nu$ over $\gS$ such that $\forall s \in A$, and $\forall B \in \Sigma$,
    \begin{equation}
        P_{F}(s, B) = \nu(B).
        \label{eq:definition-atom}
    \end{equation}
\end{definition}

The notion of atom is evident for singletons in discrete spaces, but in general Markov chains may not contain any accessible atom. Although it was not interpreted this way, we can view the augmentation of the state space in \citep{lahlou2023continuousgfn} with $\bot \notin \gS$ as a way to create an accessible artificial atom at $\bot$, if we were to also (informally) wrap around the GFlowNet at $s_{0}\equiv \bot$ as in \cref{sec:cyclical-discrete-gflownet}.

In this section, we will show how the \emph{split chain} construction \citep{nummelin1978splitting} can be used as an alternative way to create a pseudo-atom in a large class of Markov chains, without changing $\gS$. Instead of introducing a new state to $\gS$, we will use a set $\gX \in \Sigma$ that satisfies a \emph{minorization condition}. This set $\gX$ will eventually correspond to the sample space of the terminating state probability distribution induced by the GFlowNet (see \cref{def:terminating-state-probability-general}).

\begin{definition}[Minorization condition; \citealp{nummelin1978splitting}]
    Let $P_{F}$ be a $\phi$-irreducible Markov kernel over a measurable space $(\gS, \Sigma)$. A set $\gX \in \Sigma$ such that $\phi(\gX) > 0$ is said to satisfy the \emph{minorization condition} if there exists a non-negative measurable function $\varepsilon$ such that $\varepsilon^{-1}((0, +\infty)) = \gX$ (i.e., $\gX$ is the set on which $\varepsilon$ is positive), and a probability measure $\nu$, such that $\forall s\in \gS$ and $\forall B\in\Sigma$
    \begin{equation}
        P_{F}(s, B) \geq \varepsilon(s)\nu(B).%
        \label{eq:minorization-condition}%
    \end{equation}%
    \label{def:minorization-condition}%
\end{definition}
Taking $B=\gS$ in the inequality above, we can see that $\varepsilon$ is necessarily bounded, with $\varepsilon(s) \in [0, 1]$ for all $s \in \gS$. This minorization condition is not particularly interesting when $s \notin \gX$, as it simply implies that $P_{F}(s, B) \geq 0$. When $s \in \gX$ though, this allows us to interpret $P_{F}$ as a mixture of two Markov kernels, whose mixture weights depend on $\varepsilon(x)$:
\begin{equation}
    P_{F}(x, B) = \big(1 - \varepsilon(x)\big)R_{\nu}(x, B) + \varepsilon(x)\nu(B),
    \label{eq:decomposition-PF-mixture-kernels}
\end{equation}
and where $R_{\nu}(x, B)$ is a ``remainder'' Markov kernel, defined precisely in \cref{eq:remainder-kernel}; the minorization condition is necessary for this remainder kernel to exist. The important property to note is that the second kernel in this mixture, $\nu(B)$, is completely independent of $x$---only the mixture weight itself depends on $\varepsilon(x)$. We can interpret \cref{eq:decomposition-PF-mixture-kernels} as follows: we first select which kernel to apply, with probability $\varepsilon(x)$, and upon selection of the second kernel we ``reset'' the Markov chain with $\nu(B)$. Using the terminology of GFlowNets, and connecting with the Markov kernel $\bar{P}_{F}$ of \citet{lahlou2023continuousgfn}, we can view $\nu(B) \approx \bar{P}_{F}(s_{0}, B)$ as the probability of transitioning from the initial state $s_{0}$, and $\varepsilon(x) \approx \bar{P}_{F}(x, \{\bot\})$ the probability of terminating at $x \in \gX$.

\label{sec:splitting-technique}
\begin{figure}[t]
    \centering
    \begin{adjustbox}{center}
    \includegraphics[]{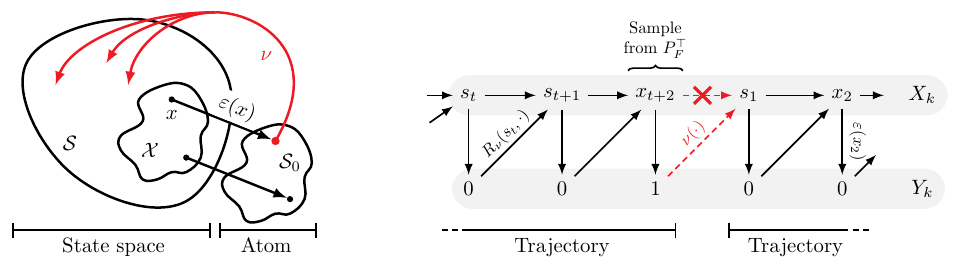}
    \end{adjustbox}
    \caption{Construction of an atom via the splitting technique. (Left) Illustration of the atom $\gS_{0} = \gX \times \{1\}$ created from $\gX$. Once the Markov chain reaches the atom $\gS_{0}$, it transitions to its next state with the kernel $\nu(\cdot)$, independent of $x$. (Right) The split chain $Z_{k} = (X_{k}, Y_{k})$, where $X_{k}$ are elements of $\gS$, and $Y_{k}$ is a binary variable indicating which kernel to choose next: if $Y_{k} = 0$, the remainder kernel $R_{\nu}(X_{k}, \cdot)$ is picked, otherwise it is $\nu(\cdot)$ indicating the ``end'' of a trajectory, with a corresponding sample from the terminating state probability $P_{F}^{\top}$.}
    \label{fig:splitting-technique}
\end{figure}

This suggests the construction of a \emph{split chain} $(Z_{k})_{k\geq 0}$, where each element can be broken down into $Z_{k} = (X_{k}, Y_{k})$, with $X_{k}$ being a state in $\gS$, and $Y_{k}$ being a binary variable indicating which of the two kernels in the mixture \cref{eq:decomposition-PF-mixture-kernels} to select at the next step. This is a Markov chain over the product space $(\gS', \Sigma') = \big(\gS \times \{0, 1\}, \Sigma \otimes \sigma(\{0, 1\})\big)$, with Markov kernel $P_{F}^{\mathrm{split}} = Q_{\nu} \otimes b_{\varepsilon}$, where
\begin{align}
    Q_{\nu}\big((x, y), B\big) &= \mathbbm{1}(y = 0)R_{\nu}(x, B) + \mathbbm{1}(y = 1)\nu(B) & B &\in \Sigma \label{eq:split-kernel-Q}\\
    b_{\varepsilon}(x, C) &= \big(1 - \varepsilon(x)\big)\delta_{0}(C) + \varepsilon(x)\delta_{1}(C) & C &\in \sigma\big(\{0, 1\}\big), \label{eq:split-kernel-b}
\end{align}
where $\delta_{y}$ is the Dirac measure at $y$. This construction is illustrated in \cref{fig:splitting-technique}. It is similar to splitting terminating states $x \in \gX$ in discrete GFlowNets into transitions $x \rightarrow x^{\top}$, where $x^{\top}$ has no children except $s_{f}$ \citep{malkin2022trajectorybalance}. It is easy to show that the set $\gX \times \{1\} \in \Sigma'$ is an atom of the split chain $(Z_{k})_{k \geq 0}$ \citep{nummelin1978splitting}. We will call this atom $\gS_{0} \triangleq \gX \times \{1\}$, by analogy with the initial state $s_{0}$ in discrete GFlowNets in \cref{sec:discrete-gflownet}, which plays a similar role.

\subsection{GFlowNets as recurrent Markov chains}
\label{sec:gflownets-recurrent-markov-chain-general}
Similar to \cref{sec:cyclical-discrete-gflownet}, we will define a GFlowNet in terms of a (Harris) recurrent Markov chain, this time over a general state space $\gS$. Instead of wrapping around the GFlowNet though to construct an atom $\{s_{0}\}$, we now use the atom $\gS_{0}$ created in \cref{sec:splitting-technique} via the splitting technique. Just like in \cref{def:terminating-state-probability} for discrete spaces, we can define the \emph{terminating state probability} over $\gS$ as the marginal distribution of the split chain returning to the atom $\gS_{0}$.

\begin{definition}[Terminating state probability]
    Let $P_{F}$ be a Harris recurrent kernel over $(\gS, \Sigma)$ such that $\gX$ satisfies the minorization condition in \cref{def:minorization-condition}. The \emph{terminating state probability distribution} is defined as the marginal distribution of the split chain returning to the atom $\gS_{0} = \gX \times \{1\}$. For all $B \in \Sigma_{\gX}$
    \begin{equation*}
        P_{F}^{\top}(B) \triangleq \E_{\gS_{0}}\big[\mathbbm{1}_{B}(X_{\sigma_{\gS_{0}}})\big] = \E_{\gS_{0}}\big[\mathbbm{1}_{B\times \{1\}}(Z_{\sigma_{\gS_{0}}})\big].
        \label{eq:terminating-state-probability-general}
    \end{equation*}
    \label{def:terminating-state-probability-general}
\end{definition}

Since $\gS_{0}$ is an atom, the kernel $\nu$ is guaranteed to be selected to transition from any state $z \in \gS_{0}$; this justifies our notation $\E_{\gS_{0}}[\cdot]$ to denote $\E_{z}[\cdot]$ for any $z \in \gS_{0}$. We show in \cref{prop:terminating-state-probability-proper-distribution-general} that $P_{F}^{\top}$ is again a properly defined probability distribution over $(\gX, \Sigma_{\gX})$, where $\Sigma_{\gX}$ is a trace $\sigma$-algebra of subsets of $\gX$. Note that in \cref{eq:terminating-state-probability-general}, we define a distribution $P_{F}^{\top}$ over $\gX$ as an expectation over the split chain $Z_{k} = (X_{k}, Y_{k})$. It is interesting to see that while the terminating state probability in discrete spaces involved the state of the Markov chain at time $\sigma_{s_{0}} - 1$ (\cref{def:terminating-state-probability}), the general case above does not have this offset by one. The reason is that in the split chain, we can treat $X_{k} \rightarrow Y_{k}$ as being a ``sub-transition'', which is enough to make up for the extra step required in the discrete case.

Since the Markov kernel $P_{F}$ is Harris recurrent, it admits an invariant measure $F$ satisfying \cref{eq:invariant-measure-general}. Just like for discrete spaces, we will require $F$ to also satisfy some \emph{boundary conditions} in order to obtain a terminating state probability $P_{F}^{\top}$ that matches some reward measure, up to normalization. For a positive measure $R$ over $\gX$, the boundary conditions take the form
\begin{equation}
    \forall B\in \Sigma_{\gX},\qquad R(B) = \int_{B}\varepsilon(x)F(dx),
    \label{eq:boundary-conditions-general}
\end{equation}
where $\varepsilon$ is the measurable function in the minorization condition. If $F$ is an invariant measure of $P_{F}$ that satisfies \cref{eq:boundary-conditions-general}, then we obtain a generalization of \cref{thm:invariant-terminating-state}.
\vspace*{2\parsep}
\begin{restatable}{ftheorem}{invariantterminatingstategeneral}
    Let $P_{F}$ be a Harris recurrent kernel over $(\gS, \Sigma)$ such that $\gX \in \Sigma$ satisfies the minorization condition in \cref{def:minorization-condition}. Moreover, assume that $P_{F}$ admits an invariant measure $F$ such that $\forall B \in \Sigma_{\gX}$,
    \begin{equation}
        R(B) = \int_{B}\varepsilon(x)F(dx),
    \end{equation}
    where $R$ is a finite measure on $\gX$ and $\varepsilon$ is the measurable function in \cref{def:minorization-condition}. Then the terminating state probability distribution is proportional to the measure $R$:
    \begin{equation*}
        \forall B \in \Sigma_{\gX}, \qquad P_{F}^{\top}(B) \propto R(B).
    \end{equation*}
    \label{thm:invariant-terminating-state-general}
\end{restatable}
The proof of this theorem, available in \cref{app:general-spaces}, relies once again on the unicity of the invariant measure of a Harris recurrent Markov chain, up to normalization.

\subsection{Construction of Harris recurrent chains}
\label{sec:constructon-harris-chains}
Since Harris recurrence is essential to guarantee that a GFlowNet does induce a terminating state probability distribution proportional to $R$, we must find ways to construct Markov kernels satisfying this property. One possibility, heavily inspired by the pointed DAG structure of discrete GFlowNets \citep{bengio2021gflownetfoundations}, is to enforce Harris recurrence through the structure of the state space. For example, the following proposition shows that any Markov chain defined on a finitely absorbing measurable pointed graph \citep{lahlou2023continuousgfn} is necessarily Harris recurrent.

\begin{restatable}{proposition}{finitelyabsorbingharris}
    Let $G$ be a finitely absorbing measurable pointed graph as defined by \citet{lahlou2023continuousgfn}, with $\kappa$ its reference transition kernel. Suppose that we identify the source state $s_{0}$ of $G$ with its sink state $\bot \equiv s_{0}$, as in \cref{sec:cyclical-discrete-gflownet}. Then any $\phi$-irreducible Markov kernel $P_{F}$ absolutely continuous wrt. $\kappa$ (in the sense that $\forall s\in \gS,\ P_{F}(s, \cdot) \ll \kappa(s, \cdot)$) is Harris recurrent.
    \label{prop:finitely-absorbing-harris}
\end{restatable}

The proof of this proposition is available in \cref{app:general-spaces}. It is important to note that in the case of measurable pointed graphs, the structure of the state space as defined by its reference transition kernel $\kappa$ is not sufficient to conclude the Harris recurrence of the chain (hence, to guarantee the fundamental theorem of GFlowNets in \cref{thm:invariant-terminating-state-general}), and that its ``finitely absorbing'' nature is essential in \citet{lahlou2023continuousgfn}. Examples of such chains include cases where the generation of an object is done in a fixed number of steps \citep{zhang2023unifying}.

However, Harris recurrence is a more general notion that goes beyond the structure of the state space as in \citet{lahlou2023continuousgfn}. For example, if $\gS \equiv \gX$, we can ensure that a $\phi$-irreducible Markov chain satisfying the minorization condition of \cref{def:minorization-condition} is Harris recurrent by enforcing $\varepsilon(s) \geq b$, for some fixed $b \in (0, 1]$ \citep{nummelin1978splitting}. In other words, this means that the GFlowNet terminates with probability at least $b > 0$ at each step of the generation.

\section{Comparison with Markov chain Monte Carlo methods}
\label{sec:comaprison-mcmc}
GFlowNets and Markov chain Monte Carlo methods (MCMC) were both introduced to solve a similar problem: sampling from a probability distribution that is defined up to a normalization constant. Applications include energy based models, where a Boltzmann distribution is defined up to its (intractable) partition function, and Bayesian inference where the posterior distribution is defined proportionally to the joint distribution, with the intractable evidence being the normalization constant. One of the main advantages of viewing GFlowNets from the perspective of Markov chains is that it places them under the same theoretical framework as MCMC methods, highlighting the similarities and differences between both methods. These differences are summarized in \cref{tab:comaprison-gflownet-mcmc}.

Recall that the goal of MCMC methods is to construct a Markov chain so that its invariant distribution matches the target distribution defined up to normalization constant. Samples from the distribution are then obtained by running the Markov chain until convergence to the invariant distribution, which typically requires to run the Markov chain for a long (burn-in) period. The convergence of iterates $\{P^{n}(s_{0}, \cdot)\}_{n \geq 0}$ to the invariant distribution is guaranteed by the ergodicity of the Markov chain (i.e. positive recurrent, \emph{and} aperiodic). By contrast, the Markov chain of a GFlowNet is only required to be positive recurrent (\cref{sec:cyclical-discrete-gflownet}) to guarantee the existence of an invariant distribution $F$, but no guarantee on the convergence to $F$ is necessary. Moreover, while the Markov kernel $P_{F}$ of MCMC methods needs to be carefully built to ensure that the invariant distribution matches the target distribution, the invariant distribution of the Markov kernel in a GFlowNet may be arbitrary, as long as it matches the boundary condition \cref{eq:boundary-condition}; there may be multiple Markov kernels with different invariant distributions yielding the same terminating state probability. This could explain why the Markov kernels in MCMC methods are typically handcrafted (e.g., Metropolis-Hastings, based on a proposal kernel), as opposed to GFlowNets where $P_{F}$ is learned (e.g., with a neural network).

\begin{table}[t]
    \centering
    \caption{Comparison between MCMC methods and GFlowNets, both defined on a discrete state space $\gS$ (for simplicity) with a positive recurrent Markov kernel $P_{F}$, with invariant measure/distribution $F$. In both cases, the objective is to approximate a target distribution $\propto R(x)$. ${}^{\dagger}$There exists MCMC methods augmenting the state space.}
    \label{tab:comaprison-gflownet-mcmc}
    \begin{tabular}{lcccc}
        \toprule
         & Sample & Conditions & Target distribution & \multirow{2}{*}{Sampling}\\
         & space & \footnotesize{(+ positive rec.)} & $\propto R(x)$ & \\
        \midrule
        \multirow{2}{*}{MCMC} & \multirow{2}{*}{$\gS^{\dagger}$} & Aperiodicity & $ = F(x)$ & Asymptotic \\
        & & \footnotesize{(convergence)} & \footnotesize{(invariant dist.)} & \footnotesize{(correlated)} \\[0.5em]
        \multirow{2}{*}{GFlowNet} & \multirow{2}{*}{$\gX \subseteq \gS$} & Boundary cond. & $ = P_{F}^{\top}(x)$ & Finite time \\
        & & \footnotesize{(marginal)} & \footnotesize{(terminating state dist.)} & \footnotesize{(independent)} \\
        \bottomrule
    \end{tabular}
\end{table}

Probably the main difference between GFlowNets and MCMC methods is the relation between the invariant measure/distribution of the Markov chain and the target distribution: while MCMC requires the invariant distribution to match the target distribution, GFlowNets only require the \emph{marginal distribution} of the Markov chain (the terminating state probability distribution; \cref{def:terminating-state-probability}) to match the target distribution. The state space of the Markov chain in MCMC therefore corresponds to the sample space of the target distribution, and as a consequence these chains are known to poorly mix (leading to slow convergence to the invariant distribution) in the presence of multiple modes. The Markov chain of a GFlowNet, on the other hand, is constructed on an augmented state space $\gS$, broader than the sample space $\gX \subseteq \gS$, allowing the chain to use these intermediate steps to move between modes more easily. It is worth noting that there exists some MCMC methods, such as \emph{Hamiltionian Monte Carlo} methods (HMC; \citealp{neal1993probabilistic,mackay2003information}), where the target distribution is the marginal of the invariant distribution over an augmented space.

Finally, since MCMC methods rely on the convergence to the invariant distribution, samples of the target distribution are only guaranteed asymptotically. Moreover, consecutive samples are correlated by the Markov kernel $P_{F}$, and additional post-processing techniques are required to reduce the effect of this cross-correlation between samples. On the other hand, samples from a GFlowNet are obtained in finite time due to the positive recurrence (or Harris recurrence in general state spaces) of the Markov chain, and are guaranteed to be independent from one another thanks to the strong Markov property.
\newpage
\section*{Acknowledgments}
\label{sec:acknowledgments}
We would like to thank Nikolay Malkin, Salem Lahlou, Pablo Lemos, and Dinghuai Zhang for the useful discussions and feedback about this paper.

\bibliographystyle{plainnat}
\bibliography{references}

\begin{thebibliography}{16}
\providecommand{\natexlab}[1]{#1}
\providecommand{\url}[1]{\texttt{#1}}
\expandafter\ifx\csname urlstyle\endcsname\relax
  \providecommand{\doi}[1]{doi: #1}\else
  \providecommand{\doi}{doi: \begingroup \urlstyle{rm}\Url}\fi

\bibitem[Bengio et~al.(2021{\natexlab{a}})Bengio, Jain, Korablyov, Precup, and
  Bengio]{bengio2021gflownet}
Emmanuel Bengio, Moksh Jain, Maksym Korablyov, Doina Precup, and Yoshua Bengio.
\newblock {Flow Network based Generative Models for Non-Iterative Diverse
  Candidate Generation}.
\newblock \emph{Neural Information Processing Systems}, 2021{\natexlab{a}}.

\bibitem[Bengio et~al.(2021{\natexlab{b}})Bengio, Lahlou, Deleu, Hu, Tiwari,
  and Bengio]{bengio2021gflownetfoundations}
Yoshua Bengio, Salem Lahlou, Tristan Deleu, Edward~J Hu, Mo~Tiwari, and
  Emmanuel Bengio.
\newblock {GFlowNet Foundations}.
\newblock \emph{arXiv preprint}, 2021{\natexlab{b}}.

\bibitem[Deleu et~al.(2022)Deleu, G{\'o}is, Emezue, Rankawat, Lacoste-Julien,
  Bauer, and Bengio]{deleu2022daggflownet}
Tristan Deleu, Ant{\'o}nio G{\'o}is, Chris Emezue, Mansi Rankawat, Simon
  Lacoste-Julien, Stefan Bauer, and Yoshua Bengio.
\newblock {Bayesian Structure Learning with Generative Flow Networks}.
\newblock \emph{Uncertainty in Artificial Intelligence}, 2022.

\bibitem[Deleu et~al.(2023)Deleu, Nishikawa-Toomey, Subramanian, Malkin,
  Charlin, and Bengio]{deleu2023jspgfn}
Tristan Deleu, Mizu Nishikawa-Toomey, Jithendaraa Subramanian, Nikolay Malkin,
  Laurent Charlin, and Yoshua Bengio.
\newblock {Joint Bayesian Inference of Graphical Structure and Parameters with
  a Single Generative Flow Network}.
\newblock \emph{arXiv preprint}, 2023.

\bibitem[Douc et~al.(2018)Douc, Moulines, Priouret, and
  Soulier]{douc2018markovchains}
Randal Douc, Eric Moulines, Pierre Priouret, and Philippe Soulier.
\newblock \emph{{Markov Chains}}.
\newblock Springer, 2018.

\bibitem[Hu et~al.(2023)Hu, Malkin, Jain, Everett, Graikos, and
  Bengio]{hu2023gfnem}
Edward~J. Hu, Nikolay Malkin, Moksh Jain, Katie Everett, Alexandros Graikos,
  and Yoshua Bengio.
\newblock {GFlowNet-EM} for learning compositional latent variable models.
\newblock \emph{International Conference on Machine Learning}, 2023.

\bibitem[Jain et~al.(2022)Jain, Bengio, Hernandez-Garcia, Rector-Brooks,
  Dossou, Ekbote, Fu, Zhang, Kilgour, Zhang, Simine, Das, and
  Bengio]{jain2022gfnbiological}
Moksh Jain, Emmanuel Bengio, Alex Hernandez-Garcia, Jarrid Rector-Brooks,
  Bonaventure~F.P. Dossou, Chanakya Ekbote, Jie Fu, Tianyu Zhang, Micheal
  Kilgour, Dinghuai Zhang, Lena Simine, Payel Das, and Yoshua Bengio.
\newblock {Biological Sequence Design with GFlowNets}.
\newblock \emph{International Conference on Machine Learning}, 2022.

\bibitem[Jain et~al.(2023)Jain, Deleu, Hartford, Liu, Hernandez-Garcia, and
  Bengio]{jain2023gfnscientific}
Moksh Jain, Tristan Deleu, Jason Hartford, Cheng-Hao Liu, Alex
  Hernandez-Garcia, and Yoshua Bengio.
\newblock {GFlowNets for AI-Driven Scientific Discovery}.
\newblock \emph{Digital Discovery}, 2023.

\bibitem[Lahlou et~al.(2023)Lahlou, Deleu, Lemos, Zhang, Volokhova,
  Hern\'{a}ndez-Garc\'{i}a, Ezzine, Bengio, and
  Malkin]{lahlou2023continuousgfn}
Salem Lahlou, Tristan Deleu, Pablo Lemos, Dinghuai Zhang, Alexandra Volokhova,
  Alex Hern\'{a}ndez-Garc\'{i}a, L\'{e}na~N\'{e}hale Ezzine, Yoshua Bengio, and
  Nikolay Malkin.
\newblock {A Theory of Continuous Generative Flow Networks}.
\newblock \emph{International Conference on Machine Learning}, 2023.

\bibitem[Li et~al.(2023)Li, Luo, Wang, and Hao]{li2023cflownets}
Yinchuan Li, Shuang Luo, Haozhi Wang, and Jianye Hao.
\newblock {CFlowNets: Continuous control with Generative Flow Networks}.
\newblock \emph{International Conference on Learning Representations}, 2023.

\bibitem[MacKay(2003)]{mackay2003information}
David~JC MacKay.
\newblock \emph{Information theory, inference and learning algorithms}.
\newblock Cambridge university press, 2003.

\bibitem[Malkin et~al.(2022)Malkin, Jain, Bengio, Sun, and
  Bengio]{malkin2022trajectorybalance}
Nikolay Malkin, Moksh Jain, Emmanuel Bengio, Chen Sun, and Yoshua Bengio.
\newblock Trajectory balance: Improved credit assignment in {GFlowNets}.
\newblock \emph{Neural Information Processing Systems}, 2022.

\bibitem[Meyn and Tweedie(1993)]{meyn1993markovchainsstability}
Sean~P Meyn and Richard~L Tweedie.
\newblock \emph{{Markov Chains and Stochastic Stability}}.
\newblock Springer-Verlag, 1993.

\bibitem[Neal(1993)]{neal1993probabilistic}
Radford~M Neal.
\newblock {Probabilistic inference using Markov chain Monte Carlo methods}.
\newblock 1993.

\bibitem[Nummelin(1978)]{nummelin1978splitting}
Esa Nummelin.
\newblock {A splitting technique for Harris recurrent Markov chains}.
\newblock \emph{Zeitschrift f{\"u}r Wahrscheinlichkeitstheorie und verwandte
  Gebiete}, 1978.

\bibitem[Zhang et~al.(2023)Zhang, Chen, Malkin, and Bengio]{zhang2023unifying}
Dinghuai Zhang, Ricky T.~Q. Chen, Nikolay Malkin, and Yoshua Bengio.
\newblock Unifying generative models with {GFlowNets} and beyond.
\newblock \emph{arXiv preprint}, 2023.

\end{thebibliography}

\newpage
\appendix
\numberwithin{equation}{section}
\numberwithin{figure}{section}
\numberwithin{table}{section}

{\LARGE \textbf{Appendix}}

\section{Existence of an invariant measure}
\label{app:existence-invariant-measure}
In this section, we recall some standard results about the existence of an invariant measure for Markov chains, first over a discrete state space, and then over general state spaces. For further fundamentals about Markov chains, we recommend the book \citep{douc2018markovchains}. We start by stating the existence of an invariant measure for irreducible and positive recurrent Markov chains over a discrete state space $\gS$, used to show the fundamental theorem of GFlowNets in \cref{thm:invariant-terminating-state}.

\begin{theorem}
    Let $P_{F}$ be an irreducible and positive recurrent Markov kernel over $\gS$. Then there exists a non-trivial invariant measure $\lambda$ for $P_{F}$ (i.e., $\lambda P_{F} = \lambda$), unique up to a multiplicative positive constant, defined for all $s\in\gS$ by:
    \begin{equation*}
        \lambda(s) = \E_{s_{0}}\!\left[\sum_{k=0}^{\sigma_{s_{0}}-1}\mathbbm{1}(X_{k} = s)\right] = \sum_{k=0}^{\infty}\E_{s_{0}}\big[\mathbbm{1}(k < \sigma_{s_{0}})\mathbbm{1}(X_{k}=s)\big]
    \end{equation*}
    In particular, $\lambda$ is the unique invariant measure of $P_{F}$ such that $\lambda(s_{0}) = 1$.
    \label{thm:unique-invariant-measure}
\end{theorem}
\begin{proof}
    See \citep[][Theorem 7.2.1]{douc2018markovchains}.
\end{proof}

When $(\gS, \Sigma)$ is a general measurable state space, we first need to introduce the notion of $\phi$-irreducibility, characterizing the sets that are accessible by the Markov chain using a measure~$\phi$.

\begin{definition}[$\phi$-irreducibility]
    Let $\phi$ be a measure over a measurable space $(\gS, \Sigma)$. A Markov kernel $P_{F}$ is said to be \emph{$\phi$-irreducible} if any set $B \in \Sigma$ such that $\phi(B) > 0$ is accessible, in the sense that for all $s\in\gS$
    \begin{equation}
        \sP_{s}(\tau_{B} < \infty) > 0,
    \end{equation}
    where $\tau_{B} = \inf \{n \geq 0\mid X_{k} \in B\}$ is the hitting time of the Markov chain to $B$.
    \label{def:phi-irreducibility}
\end{definition}

We can extend \cref{thm:unique-invariant-measure} guaranteeing the existence of an invariant measure to Harris recurrent kernels satisfying the minorization condition. Although invariant measures may also exist for Markov kernels under weaker assumptions, here we use the assumptions that match the conditions necessary for GFlowNets over measurable state spaces to exist.

\begin{theorem}
    Let $P_{F}$ be a Harris recurrent kernel over $\gS$ such that $\gX$ satisfies the minorization condition in \cref{def:minorization-condition}. Then there exists a non-trivial invariant measure $\lambda$ for $P_{F}$ (i.e., $\lambda P_{F} = \lambda$), unique up to a multiplicative positive constant, defined for all bounded measurable functions $f: \gS \rightarrow \sR$ by
    \begin{equation*}
        \int_{\gS}f(s)\lambda(ds) = \E_{\gS_{0}}\Bigg[\sum_{k=1}^{\sigma_{\gS_{0}}}f(X_{k})\Bigg] = \sum_{k=1}^{\infty}\E_{\gS_{0}}\big[\mathbbm{1}(k \leq \sigma_{\gS_{0}})f(X_{k})\big].
    \end{equation*}
    Moreover, $\lambda$ is the unique measure such that $\int_{\gS}\varepsilon(s)\lambda(ds) = 1$, where $\varepsilon$ is the measurable function in the minorization condition.
    \label{thm:unique-invariant-measure-general}
\end{theorem}

\begin{proof}
    The existence and the form of the invariant measure are given in \citep[][Theorem 3]{nummelin1978splitting}. It is therefore sufficient to prove that $\int_{\gS}\varepsilon(s)\lambda(ds) = 1$. We first show that $\forall k \geq 1$:
    \begingroup
    \allowdisplaybreaks
    \begin{align}
        \E_{\gS_{0}}\big[\mathbbm{1}(k \leq \sigma_{\gS_{0}})\varepsilon(X_{k})\big] &= \E_{\gS_{0}}\big[\mathbbm{1}(k \leq \sigma_{\gS_{0}})\E_{Z_{k-1}}[\mathbbm{1}(Y_{1} = 1)\mid X_{1}]\big]\label{eq:proof-unique-invariant-measure-1}\\
        &= \E_{\gS_{0}}\big[\mathbbm{1}(k \leq \sigma_{\gS_{0}})\E_{\gS_{0}}[\mathbbm{1}(Y_{k}=1)\mid Z_{0:k-1},X_{k}]\big]\label{eq:proof-unique-invariant-measure-2}\\
        &= \E_{\gS_{0}}\big[\mathbbm{1}(k \leq \sigma_{\gS_{0}})\mathbbm{1}(Y_{k}=1)\big]\label{eq:proof-unique-invariant-measure-3}\\
        &= \E_{\gS_{0}}\big[\mathbbm{1}(\sigma_{\gS_{0}} = k)\big]\label{eq:proof-unique-invariant-measure-4}
    \end{align}
    \endgroup
    where we used the interpretation of $\varepsilon(x)$ as the probability of terminating at $x$ (see \cref{sec:splitting-technique}) in \cref{eq:proof-unique-invariant-measure-1}, the Markov property of the chain $(Z_{k})_{k\geq 0}$ in \cref{eq:proof-unique-invariant-measure-2}, the law of total expectation in \cref{eq:proof-unique-invariant-measure-3}, and finally the definition of $\sigma_{\gS_{0}}$ as being the return time to $\gS_{0}$, where we necessarily have $Y = 1$ in \cref{eq:proof-unique-invariant-measure-4}. Therefore, using the for of the invariant measure $\lambda$, and since $\varepsilon$ is a measurable function, we have
    \begin{equation}
        \int_{\gS}\varepsilon(s)\lambda(ds) = \sum_{k=1}^{\infty}\E_{\gS_{0}}\big[\mathbbm{1}(k \leq \sigma_{\gS_{0}})\varepsilon(X_{k})\big] = \sum_{k=1}^{\infty}\E_{\gS_{0}}\big[\mathbbm{1}(\sigma_{\gS_{0}} = k)\big] = 1,
    \end{equation}
    where we were able to conclude since the chain is Harris recurrent (hence it returns to $\gS_{0}$ with probability $1$; \cref{def:harris-recurrence}).
\end{proof}

If the Markov kernel $P_{F}$ satsifies the minorization condition in \cref{def:minorization-condition}, then we showed in \cref{sec:splitting-technique} that we can write $P_{F}$ as a mixture of two Markov kernels, one being $\nu(B)$ being independent of the current state, and the other being a \emph{remainder kernel} $R_{\nu}(x, B)$, defined by
\begin{equation}
    R_{\nu}(x, B) = \mathbbm{1}(\varepsilon(x) < 1)\frac{P_{F}(x, B) - \varepsilon(x)\nu(B)}{1 - \varepsilon(x)} + \mathbbm{1}(\varepsilon(x) = 1)\nu(B)
    \label{eq:remainder-kernel}
\end{equation}
One can easily show that this is a Markov kernel, and can be defined thanks to the minorization condition $P_{F}(x, B) \geq \varepsilon(x)\nu(B)$.

\section{Counter-example for positive recurrence}
\label{app:counter-example-positive-recurrence}
Let state space $\gS = \sN$ be the space of non-negative integers, and $P_{F}$ the transition probability distribution defined as $\forall n \in \sN$:
\begin{align}
    P_{F}(n, n + 1) &= \exp\left[-\frac{1}{(n+1)^{2}}\right] &&& P_{F}(n, 0) &= 1 - \exp\left[-\frac{1}{(n+1)^{2}}\right],
\end{align}
with all other transitions having probability $0$. Then we have for all $n \geq 1$
\begin{align}
    \E_{0}\big[\mathbbm{1}(\sigma_{0} = n)\big] &= P_{F}(n, 0)\prod_{k=0}^{n-1}P_{F}(k, k+1) = \left(1 - \exp\left[-\frac{1}{(n+1)^{2}}\right]\right)\exp\left[-\sum_{k=0}^{n-1}\frac{1}{(k+1)^{2}}\right]\nonumber\\
    &= \exp\left[-\sum_{k=0}^{n-1}\frac{1}{(k+1)^{2}}\right] - \exp\left[-\sum_{k=0}^{n}\frac{1}{(k+1)^{2}}\right]
\end{align}
Therefore, the probability of returning to the state $0$ in finite time is
\begin{equation}
    \sP_{0}(\sigma_{0} < \infty) = \sum_{n=1}^{\infty}\E_{0}\big[\mathbbm{1}(\sigma_{0} = n)\big] = 1 - \exp\left[-\sum_{k=0}^{\infty}\frac{1}{(k+1)^{2}}\right] = 1 - \exp\left(-\frac{\pi^{2}}{6}\right) < 1
\end{equation}

\section{Proofs}
\label{app:proofs}

\subsection{Discrete spaces}
\label{app:discrete-spaces}

\begin{lemma}
    The terminating state probability distribution $P_{F}^{\top}$ is related to the invariant measure $\lambda$ defined in \cref{thm:unique-invariant-measure} by, $\forall x\in \gX$, $P_{F}^{\top}(x) = \lambda(x)P_{F}(x, s_{0})$.
    \label{lem:relation-unique-invariant-terminating-state-probability}
\end{lemma}
\begin{proof}
    For any $k \geq 0$, we have
    \begin{align}
        \E_{s_{0}}\big[\mathbbm{1}(k < \sigma_{s_{0}})\mathbbm{1}(X_{k} = x)\big]&P_{F}(x, s_{0}) = \E_{s_{0}}\big[\mathbbm{1}(k < \sigma_{s_{0}})\mathbbm{1}(X_{k}=x)P_{F}(X_{k}, s_{0})\big]\label{eq:proof-relation-1}\\
        &= \E_{s_{0}}\big[\mathbbm{1}(k<\sigma_{s_{0}})\mathbbm{1}(X_{k}=x)\E_{X_{k}}[\mathbbm{1}(X_{1}=s_{0})]\big]\label{eq:proof-relation-2}\\
        &= \E_{s_{0}}\big[\mathbbm{1}(k<\sigma_{s_{0}})\mathbbm{1}(X_{k}=x)\E_{s_{0}}[\mathbbm{1}(X_{k+1}=s_{0})\mid X_{0:k}]\big]\label{eq:proof-relation-3}\\
        &= \E_{s_{0}}\big[\mathbbm{1}(k<\sigma_{s_{0}})\mathbbm{1}(X_{k}=x)\mathbbm{1}(X_{k+1}=s_{0})\big]\label{eq:proof-relation-4}\\
        &= \E_{s_{0}}\big[\mathbbm{1}(\sigma_{s_{0}}=k+1)\mathbbm{1}(X_{k}=x)]\label{eq:proof-relation-5}
    \end{align}
    In details, we used an equivalent definition of $P_{F}(X_{k}, s_{0})$ in \cref{eq:proof-relation-2}, as the expectation of the chain moving to $s_{0}$ after one step ($\mathbbm{1}(X_{1}=s_{0})$), starting at $X_{k}$, the Markov property in \cref{eq:proof-relation-3}, the law of total expectation in \cref{eq:proof-relation-4}, and finally the definition of the return time $\sigma_{s_{0}} = \inf\{n\geq 1\mid X_{n}=s_{0}\}$ in \cref{eq:proof-relation-5}. Using the definition of the invariant measure $\lambda$, we then obtain the expected result:
    \begin{align}
        \lambda(x)P_{F}(x, s_{0}) &= \sum_{k=0}^{\infty}\E_{s_{0}}\big[\mathbbm{1}(k < \sigma_{s_{0}})\mathbbm{1}(X_{k}=s)\big]P_{F}(x, s_{0})\\
        &= \sum_{k=0}^{\infty}\E_{s_{0}}\big[\mathbbm{1}(\sigma_{s_{0}}=k+1)\mathbbm{1}(X_{k}=x)\big]\\
        &= \sum_{k=1}^{\infty}\E_{s_{0}}\big[\mathbbm{1}(\sigma_{s_{0}}=k)\mathbbm{1}(X_{k-1}=x)\big]\\
        &= \E_{s_{0}}\big[\mathbbm{1}(X_{\sigma_{s_{0}}-1}=x)\big] = P_{F}^{\top}(x).
    \end{align}
\end{proof}

\begin{proposition}
    The terminating state probability distribution $P^{\top}_{F}$ is a properly defined probability distribution over $\gX$.
    \label{prop:terminating-state-probability-proper-distribution}
\end{proposition}

\begin{proof}
    It is easy to see that for all $x\in\gX$, $P_{F}^{\top}(x) \geq 0$. Moreover, using \cref{lem:relation-unique-invariant-terminating-state-probability}, we know that $P_{F}^{\top}(x) = \lambda(x)P_{F}(x, s_{0})$, where $\lambda$ is the unique invariant measure of $P_{F}$ such that $\lambda(s_{0}) = 1$ (see \cref{thm:unique-invariant-measure}). Therefore:
    \begin{equation*}
        \sum_{x\in\gX}P_{F}^{\top}(x) = \sum_{x\in\gX}\lambda(x)P_{F}(x, s_{0}) = \lambda(s_{0}) = 1,
    \end{equation*}
    where we used the invariance of $\lambda$ in the second equality, and the fact that $P_{F}(s, s_{0}) = 0$ for any $s \notin \gX$.
\end{proof}

\invariantterminatingstate*

\begin{proof}
    Since $F$ is an invariant measure of $P_{F}$, by unicity of the invariant measure of $P_{F}$ up to a multiplicative constant (\cref{thm:unique-invariant-measure}), there exists a constant $\alpha > 0$ such that $F = \alpha \lambda$. Using \cref{lem:relation-unique-invariant-terminating-state-probability} and the boundary condition $F(s)P_{F}(s, s_{0}) = R(s)$, we get for all $x \in \gX$
    \begin{equation*}
        P_{F}^{\top}(x) = \lambda(x)P_{F}(x, s_{0}) = \frac{1}{\alpha}F(x)P_{F}(x, s_{0}) = \frac{R(x)}{\alpha}. 
    \end{equation*}
    In fact, since we saw in \cref{prop:terminating-state-probability-proper-distribution} that $P_{F}^{\top}$ is a probability distribution, the multiplicative constant happens to be $\alpha = R(\gX) = \sum_{x'\in\gX}R(x')$ (i.e., the partition function).
\end{proof}

\begin{proposition}
    Let $P_{F}$ be an irreducible and positive recurrent Markov kernel over $\gS$ that admits an invariant measure $F$ such that $\forall x \in \gS$, $F(s)P_{F}(s, s_{0})=R(s)$, where $R$ is a finite measure on $\gX \subseteq \gS$. Then we have $F(s_{0}) = R(\gX) = \sum_{x\in\gX}R(x)$.
\end{proposition}
\begin{proof}
    Using the boundary condition and the invariance of the measure $F$:
    \begin{equation*}
        R(\gX) = \sum_{x\in\gX}R(x) = \sum_{x\in\gX}F(x)P_{F}(x, s_{0}) = F(s_{0}).
    \end{equation*}
\end{proof}

\invariantmeasurenoinitialstate*

\begin{proof}
    One can show (e.g., by induction) that if \cref{eq:invariant-measure-no-initial-state} is satisfied for all $s'\neq s_{0}$ then
    \begin{equation}
        F(s) = \sum_{k=0}^{\infty}F(s_{0})\E_{s_{0}}\big[\mathbbm{1}(k < \sigma_{s_{0}})\mathbbm{1}(X_{k} = s)\big]
    \end{equation}
    Since $P_{F}$ is irreducible and positive recurrent, by \cref{thm:unique-invariant-measure} it admits an invariant measure. The only non-trivial statement one must show is that if \cref{eq:invariant-measure-no-initial-state} is satisfied for any $s'\neq s_{0}$, then it is also satisfied for $s' = s_{0}$. Since $P_{F}$ is positive recurrent, any Markov chain starting at $s_{0}$ must eventually return to $s_{0}$ in finite time. In other words
    \begin{equation}
        \sum_{k=1}^{\infty}\E_{s_{0}}\big[\mathbbm{1}(\sigma_{s_{0}}=k)\big] = 1.
    \end{equation}
    Furthermore, it is clear that at any point in time $k \geq 0$, $X_{k}$ must be in one of the states of $\gS$, meaning that $\sum_{s\in\gS} \mathbbm{1}(X_{k}=s) = 1$. Therefore
    \begingroup
    \allowdisplaybreaks
    \begin{align}
        \sum_{s\in\gS}F(s)P_{F}(s, s_{0}) &= \sum_{s\in\gS}\sum_{k=0}^{\infty}F(s_{0})\E_{s_{0}}\big[\mathbbm{1}(k < \sigma_{s_{0}})\mathbbm{1}(X_{k}=s)\big]P_{F}(s, s_{0})\\
        &= F(s_{0})\sum_{s\in\gS}\sum_{k=1}^{\infty}\E_{s_{0}}\big[\mathbbm{1}(\sigma_{s_{0}} = k)\mathbbm{1}(X_{k-1}=s)\big]\label{eq:proof-invariant-measure-no-initial-state-2}\\
        &= F(s_{0})\sum_{k=1}^{\infty}\E_{s_{0}}\Bigg[\mathbbm{1}(\sigma_{s_{0}} = k)\sum_{s\in\gS}\mathbbm{1}(X_{k-1}=s)\Bigg]\\
        &= F(s_{0})\sum_{k=1}^{\infty}\E_{s_{0}}\big[\mathbbm{1}(\sigma_{s_{0}}=k)\big] = F(s_{0}),
    \end{align}
    \endgroup
    where we used the proof of \cref{lem:relation-unique-invariant-terminating-state-probability} in \cref{eq:proof-invariant-measure-no-initial-state-2}.
\end{proof}

\subsection{General spaces}
\label{app:general-spaces}

\begin{lemma}
    The terminating state probability distribution $P_{F}^{\top}$ is related to the invariant measure $\lambda$ defined in \cref{thm:unique-invariant-measure-general} by, $\forall B \in \Sigma_{\gX}$,
    \begin{equation*}
        P_{F}^{\top}(B) = \int_{B}\varepsilon(x)\lambda(dx),
    \end{equation*}
    where $\varepsilon$ is the positive measurable function defined in \cref{def:minorization-condition}.
    \label{lem:relation-unique-invariant-terminating-state-probability-general}
\end{lemma}
\begin{proof}
    The proof follows the same pattern as the proof of \cref{lem:relation-unique-invariant-terminating-state-probability}. We first show that for any $k \geq 1$, we have
    \begin{align}
        \E_{\gS_{0}}\big[\mathbbm{1}(k \leq \sigma_{\gS_{0}})\mathbbm{1}_{B}(X_{k})&\varepsilon(X_{k})\big] = \E_{\gS_{0}}\big[\mathbbm{1}(k \leq \sigma_{\gS_{0}})\mathbbm{1}_{B}(X_{k})\E_{Z_{k-1}}[\mathbbm{1}(Y_{1}=1)\mid X_{1}]\big]\\
        &= \E_{\gS_{0}}\big[\mathbbm{1}(k \leq \sigma_{\gS_{0}})\mathbbm{1}_{B}(X_{k})\E_{\gS_{0}}[\mathbbm{1}(Y_{k}=1)\mid Z_{0:k-1},X_{k}]\big]\\
        &= \E_{\gS_{0}}\big[\mathbbm{1}(k \leq \sigma_{\gS_{0}})\mathbbm{1}_{B}(X_{k})\mathbbm{1}(Y_{k}=1)\big]\\
        &= \E_{\gS_{0}}\big[\mathbbm{1}(\sigma_{\gS_{0}}=k)\mathbbm{1}_{B}(X_{k})\big]
    \end{align}
    The derivation above follows similar steps as in the proofs of \cref{thm:unique-invariant-measure-general} \& \cref{lem:relation-unique-invariant-terminating-state-probability}. Using the definition of the invariant measure $\lambda$, and given that for any $B\in \Sigma_{\gX}$ the measurable function $\mathbbm{1}_{B}\varepsilon$ is non-negative, we have
    \begin{align}
        \int_{B}\varepsilon(x)\lambda(dx) &= \sum_{k=1}^{\infty}\E_{\gS_{0}}\big[\mathbbm{1}(k \leq \sigma_{\gS_{0}})\mathbbm{1}_{B}(X_{k})\varepsilon(X_{k})\big]\\
        &= \sum_{k=1}^{\infty}\E_{\gS_{0}}\big[\mathbbm{1}(\sigma_{\gS_{0}} = k)\mathbbm{1}_{B}(X_{k})\big]\\
        &= \E_{\gS_{0}}\big[\mathbbm{1}_{B}(X_{\sigma_{\gS_{0}}})\big] = P_{F}^{\top}(B).
    \end{align}
\end{proof}

\begin{proposition}
    The terminating state probability distribution $P_{F}^{\top}$ defined in \cref{def:terminating-state-probability-general} is a properly defined probability distribution over $\gX$.
    \label{prop:terminating-state-probability-proper-distribution-general}
\end{proposition}

\begin{proof}
    Using \cref{lem:relation-unique-invariant-terminating-state-probability-general}, we know that $P_{F}^{\top}$ is related to the unique invariant measure $\lambda$ such that $\int_{\gS}\varepsilon(s)\lambda(ds) = 1$. Therefore:
    \begin{equation}
        P_{F}^{\top}(\gX) = \int_{\gX}\varepsilon(x)\lambda(dx) = \int_{\gS}\varepsilon(s)\lambda(ds) = 1,
    \end{equation}
    where we used the fact that $\varepsilon$ is positive only on $\gX$ (see \cref{def:minorization-condition}).
\end{proof}

\invariantterminatingstategeneral*

\begin{proof}
    The proof is similar to the one of \cref{thm:invariant-terminating-state}. By unicity of the invariant measure of $P_{F}$ up to a multiplicative constant (\cref{thm:unique-invariant-measure-general}), there exists a constant $\alpha > 0$ such that $F = \alpha \lambda$. Using \cref{lem:relation-unique-invariant-terminating-state-probability-general}, and the boundary conditions in \cref{eq:boundary-conditions-general}, we get for all $B \in \Sigma_{\gX}$:
    \begin{equation}
        P_{F}^{\top}(B) = \int_{B}\varepsilon(x)\lambda(dx) = \frac{1}{\alpha}\int_{B}\varepsilon(x)F(dx) = \frac{R(B)}{\alpha} \propto R(B).
    \end{equation}
\end{proof}

\finitelyabsorbingharris*

\begin{proof}
    Let $(\gS, \gT)$ be the topological space over which the measurable pointed graph $G$ is defined. Recall that a measurable pointed graph is finitely absorbing in $\exists N > 0$ such that
    \begin{equation}
        \mathrm{supp}\big(\kappa^{N}(s_{0}, \cdot)\big) = \{\bot\}.
        \label{eq:finitely-absorbing}
    \end{equation}
    We assume that $N$ is the minimal integer satisfying this property (i.e., the maximal trajectory length). Let $P_{F}$ be a Markov kernel absolutely continuous wrt. $\kappa$. Then we necessarily have that $\mathrm{supp}\big(P_{F}^{N}(s_{0}, \cdot)\big) = \{\bot\}$ as well. Let $B \in \gT$ be a set such that $\phi(B) > 0$. It is clear that there exists $0 \leq m \leq N$ such that $\forall s \in B, \mathrm{supp}\big(P_{F}^{m}(s, \cdot)\big) = \{\bot\}$ (otherwise, this would contradict the maximal trajectory length).
    
    By definition of a measurable pointed graph, $\exists n \geq 0, \kappa^{n}(s_{0}, B) > 0$ (and necessarily $P_{F}^{n}(s_{0}, B) > 0$ as well). If we identify the source and sink states $s_{0} \equiv \bot$, then by Chapman-Kolmogorov equation we have $\forall s \in B$:
    \begin{equation}
        P_{F}^{n+m}(s, B) = \int_{\gS}P_{F}^{m}(s, ds')P_{F}^{n}(s', B) = P_{F}^{m}(s, \{\bot\})P_{F}^{n}(s_{0}, B) > 0.
    \end{equation}
    Moreover, by minimality of $N$, we also have that $\forall n > N, P_{F}^{n}(s, B) = 0$. Therefore,
    \begin{equation}
        \forall s \in B,\quad \sP_{s}(\sigma_{B} \leq N < \infty) = 1.
    \end{equation}
\end{proof}

\end{document}